\documentclass{article}

\usepackage[preprint]{neurips_2026}

\usepackage[utf8]{inputenc}
\usepackage[T1]{fontenc}
\usepackage{hyperref}
\usepackage{url}
\usepackage{booktabs}
\usepackage{amsfonts}
\usepackage{amsmath}
\usepackage{amssymb}
\usepackage{amsthm}
\usepackage{nicefrac}
\usepackage{microtype}
\usepackage{xcolor}
\usepackage{graphicx}
\graphicspath{{./}}
\usepackage{wrapfig}
\usepackage{subcaption}
\usepackage{multirow}
\usepackage{tabularx}
\usepackage{enumitem}
\usepackage{CJKutf8}

\newtheorem{theorem}{Theorem}

\newtheorem{definition}{Definition}

\title{When Think-with-Image Meets Safety: What Determines Multimodal Jailbreak Robustness?}

\author{%
  Yuan Tian\thanks{Equal contribution.}\thanks{Corresponding authors: \texttt{yuan.tian.research@gmail.com}, \texttt{binghuresearch@gmail.com}.} \\
  Independent Researcher
  \And
  Bing Hu\footnotemark[1]\footnotemark[2] \\
  Independent Researcher
  \AND
  Fang Wu \\
  Stanford University
  \And
  Xiaomin Li \\
  Harvard University
  \AND
  Binghang Lu \\
  Purdue University
  \And
  Neil Zhenqiang Gong \\
  Duke University
}

\begin{document}

\maketitle
\begin{abstract}
Think-with-image reasoning is emerging as a new inference paradigm for large vision-language models, but its safety implications remain poorly understood. Existing systems already span multiple process designs, including direct response generation, text-only prior turn, visual-state manipulation, and explicit external image-tool invocation. In this paper, we ask which of these evaluated paradigms improves multimodal jailbreak robustness, and why. Across multiple vision-language models, explicit image-tool interaction yields the lowest attack success rates in our experiments, reducing jailbreak success by around 30\% relative on average across the evaluated models. This finding is initially surprising: ASR remains low even when the returned image-tool output is manually overridden or itself unsafe-looking, but returns near direct-answering levels under text-only prior turn controls. These results indicate that the lower ASR is not explained by benign returned-image semantics or by the textual image-tool trace alone. To explain the pattern, we introduce an \emph{image-tool safety vector} framework that models image-tool invocation as a residual shift in hidden representations toward a safety-relevant direction. Representation-level analyses and activation interventions support this account. Overall, our results suggest that explicit image-tool interaction is a promising design pattern for improving jailbreak robustness, while also motivating pipeline-specific safety evaluation.
\end{abstract}

\section{Introduction}

Large vision-language models (LVLMs) are rapidly moving from single-pass image-text inference toward \emph{think-with-image} (TwI) reasoning, where natural-language reasoning is interleaved with intermediate visual operations~\citep{wu2023visualchatgpt,suris2023vipergpt,yang2023mmreact,wu2024vstar,su2025thinkimages}. Existing systems span paradigms ranging from direct image generation or editing to grounding intermediate steps to regions, sketches, or coordinates~\citep{yang2023setofmark,peng2024kosmos2,gupta2023visprog}, and frameworks such as Visual Sketchpad~\citep{hu2024visual} and Chain of Multi-modal Thought~\citep{cheng2025comt} further augment LVLMs with external tools, including object detectors, segmenters, and geometric solvers, whose outputs are fed back into later reasoning~\citep{yao2023react,schick2023toolformer,lu2023chameleon,chen2026towards}. This shift has expanded multimodal capability on tasks requiring spatial grounding, compositional perception, and multi-step decision making.

However, think-with-image reasoning is not only a capability shift; it is also a structural change to the inference process. Different paradigms introduce different intermediate states, control flows, and external interactions. Compared with direct query-response generation, image-tool-augmented interleaving inserts additional decision points: the model decides whether to call an image tool, produces a structured image-tool request, consumes a returned visual artifact, and conditions later reasoning on that artifact. These interaction steps introduce token patterns and visual states that are absent from standard end-to-end inference. As a result, the safety behavior of a TwI system cannot be inferred directly from the safety properties of the base model alone, and different paradigms may have materially different safety profiles.

This question is increasingly important as multimodal systems are deployed in settings where visual reasoning and tool use are closely intertwined~\citep{zheng2024seeact,xie2024osworld,koh2024visualwebarena}. Prior work shows that LVLMs remain vulnerable to harmful image-text inputs~\citep{liu2024mmsafetybench,luo2024jailbreakv}, while tool-use safety studies have mainly examined text-only agents or risks introduced by external tools~\citep{greshake2023indirect,ruan2024toolemu,zhan2024injecagent,debenedetti2024agentdojo}. Yet it remains unclear how different TwI paradigms affect safety behavior: explicit image-tool interaction may amplify multimodal jailbreak risk, shift failure modes, or even improve robustness.

We study this question across multiple TwI paradigms and vision-language models, with a particular focus on explicit image-tool invocation. Our main empirical finding is counterintuitive: \textbf{explicit image-tool interaction yields the lowest attack success rates among the evaluated settings.} Across our evaluations, this paradigm reduces jailbreak success by around 30\% relative on average across the evaluated models. The effect is robust to the visual content returned by the image tool: even when the returned image is manually overridden or itself unsafe, models remain harder to jailbreak as long as the reasoning process includes explicit image-tool interaction. This result suggests that the lower ASR is not explained solely by the semantics of the returned image. Instead, the image-tool-calling trajectory appears to act as a safety-relevant cue during TwI.

To account for this phenomenon, we introduce an \emph{image-tool safety vector} framework. At a representation level, the framework models image-tool invocation as a residual shift in hidden states and asks whether the resulting image-tool state makes safety behavior more linearly readable and more steerable. Representation-level analyses across returned-image settings support this view and provide a mechanistic account for why explicit image-tool interaction can improve safety even when the image-tool output itself is not benign.

Our contributions are threefold:
\begin{itemize}[leftmargin=*,topsep=0pt,itemsep=0pt]
    \item We frame multimodal jailbreak robustness as a comparison across think-with-image process paradigms, including direct response, text-only prior turn, simulated image generation/editing, visual-state variants, and explicit image-tool interaction.
    \item We find that explicit image-tool interaction yields the lowest attack success rates among the evaluated settings and show that the ASR reduction remains under substantial changes to returned visual content but disappears under text-only prior turn controls.
    \item We introduce the image-tool safety vector framework and support it with representation-level and activation-intervention evidence, linking image-tool invocation to an internal state that is more safety-separable and behaviorally steerable.
\end{itemize}

\vspace{-1mm}\section{Related Work}\vspace{-2mm}

Our study connects think-with-image reasoning, tool-using multimodal agents, multimodal jailbreak evaluation, and representation-level safety analysis.\vspace{-8pt}

\paragraph{Think-with-image reasoning.}
Recent LVLMs increasingly perform \emph{think-with-image} reasoning, where intermediate visual states become part of inference. Prior work externalizes visual thoughts~\citep{cheng2026visual} through sketches, whiteboards, generated visualizations, regions, or coordinates~\citep{hu2024visual,menon2024whiteboard,li2025vot,yang2023setofmark,wu2024vstar}, and also studies reasoning boundaries, multi-step multimodal reasoning benchmarks, or training procedures~\citep{chen2024unlocking,cheng2025comt,chen2024m3cot,su2025openthinkimg}. These works establish think-with-image as a broad capability paradigm, but largely leave its adversarial safety implications unexplored.\vspace{-8pt}

\paragraph{Tool-augmented multimodal agents.}
Tool-using models interleave reasoning with external actions or modules~\citep{yao2023react,schick2023toolformer}. Multimodal agents extend this paradigm by integrating vision tools, executable programs, and grounded actions for visual reasoning and agentic tasks~\citep{wu2023visualchatgpt,gupta2023visprog,suris2023vipergpt,lu2023chameleon,zheng2024seeact,xie2024osworld}. Existing safety studies mainly focus on external failures, including indirect prompt injection, unsafe actions, and environment-level risks~\citep{greshake2023indirect,ruan2024toolemu,zhan2024injecagent,debenedetti2024agentdojo}, leaving underexplored how tool invocation reshapes a model's internal safety behavior.\vspace{-8pt}

\paragraph{Multimodal jailbreak evaluation.}
Multimodal safety benchmarks show that harmful image-text inputs remain effective against LVLMs~\citep{liu2024mmsafetybench,luo2024jailbreakv}, while broader benchmarks standardize jailbreak and refusal evaluation~\citep{mazeika2024harmbench,chao2024jailbreakbench,weng2025mmjbench}. This line provides threat models and metrics, but typically evaluates a fixed inference pathway rather than treating the reasoning paradigm as an experimental variable.\vspace{-8pt}

\paragraph{Representation-level safety analysis.}
Representation-level work shows that refusal and safety behavior can often be detected or influenced through low-dimensional activation patterns~\citep{zou2023repeng,arditi2024refusal,lee2025conditional,wollschlager2025geometry,xu2024scav,qian2025hsf,wang2025astra}; concurrent work extends this to large vision-language models, both for activation steering~\citep{yang2026steering} and for steering safety judgments via semantic cues~\citep{hinojosa2026saves}. This motivates our analysis of whether explicit image-tool interaction induces a consistent safety-relevant shift in LVLM hidden states.

Prior work leaves open whether think-with-image paradigms themselves change jailbreak robustness. We study this directly by comparing direct, simulated image generation/editing, visual-state, and explicit image-tool interaction settings under the same safety protocol. This reveals a surprising pattern: explicit image-tool interaction produces the lowest attack success rates in our experiments, and the reduction points beyond returned-image semantics to the image-tool-calling process within multimodal reasoning.

\vspace{-1mm}\section{Preliminary}\vspace{-1mm}
\subsection{We Compare Paradigms by Holding the Safety Query Fixed}\vspace{-2mm}
Let $q_i=(x_i,t_i)$ denote the $i$-th safety query, where $x_i$ is the adversarial image and $t_i$ is the harmful text prompt. For each paradigm $k$, the prefix $\pi_k$ represents the process that happens before the safety query, while $q_i$ itself stays fixed. Given an LVLM $f_\theta$, an unsafe-response judge $J$, and $N$ evaluation items, we compute the attack success rate (ASR) under paradigm $k$:
\begin{equation}
y_{i,k}=f_\theta(\pi_k\oplus q_i),
\qquad
\widehat{\mathrm{ASR}}_k=\frac{1}{N}\sum_{i=1}^{N}J(y_{i,k}),
\end{equation}
Here, $\oplus$ denotes concatenating the prefix with the later safety query, $y_{i,k}$ is the model's final answer under paradigm $k$, and $J(y_{i,k})=1$ means that this answer is judged unsafe.

\vspace{-1mm}\subsection{Paradigms Differ in the Process Before the Safety Query}\vspace{-2mm}
We define a paradigm by the prefix placed before the same MM-SafetyBench query. Each paradigm corresponds to a different think-with-image process, with different intermediate states and interactions. The main paradigm comparison is summarized in Table~\ref{tab:paradigm-overview}, and the text-only prior turn control is reported separately in Table~\ref{tab:text-prefill}. \textbf{Direct baseline} asks the model to answer the safety query immediately, so $\pi_{\mathrm{direct}}=\emptyset$. \textbf{Text-only prior turn} checks whether a completed prior textual image-tool-like turn is enough to change safety behavior; $\pi_{\mathrm{text}}$ keeps a completed text-only prefix but removes the returned image and any subsequent visual feedback. \textbf{Simulated image generation/editing} tests whether a generated or edited visual state before the safety query matters without image-tool-based analysis; $\pi_{\mathrm{gen}}$ contains a generated prefill image. \textbf{Visual-state variants} test the role of returned-image content; $\pi_{\mathrm{state}}=(b,a,r')$ keeps a benign prefix task $b$ and the model's image-tool request $a$ but replaces the returned artifact with a controlled image $r'$. \textbf{Explicit image-tool interaction} tests the full image-tool-mediated process; $\pi_{\mathrm{tool}}=(b,a,T(a))$ contains the benign task, the model's image-tool request, and the actual VSP output $T(a)$ returned by the image tool.
In all non-direct settings, the adversarial image $x_i$ is absent from $\pi_k$ and appears only inside $q_i$; the judge scores only the final answer $y_{i,k}$.

\vspace{-1mm}\section{Experimental Setting}\vspace{-2mm}

We test whether the \emph{process} preceding a harmful multimodal query affects jailbreak robustness. Each instance pairs a fixed safety query with, when applicable, a benign think-with-image prefix that induces intermediate visual computation.
For safety evaluation, we use MM-SafetyBench~\citep{liu2024mmsafetybench}, whose adversarial image-text pairs are generated via typographic attacks or Stable Diffusion~\citep{rombach2022ldm} across 13 safety-sensitive categories. Due to the cost of think-with-image pipelines, we randomly sample 202 questions using a fixed sampling seed; the per-category composition of this subset is given in Appendix~\ref{app:mmsb-categories}. The primary metric is attack success rate (ASR), defined as the fraction of final safety-query responses judged unsafe by an LLM-as-judge evaluator~\citep{zheng2023judging}. In multi-turn settings, the judge is given only the model's final answer to the MM-SafetyBench query, not the benign prefix or image-tool transcript, so the prefix affects ASR only through the answer produced by the evaluated model.

For the benign prefix, we use CoMT tasks~\citep{cheng2025comt} that reliably elicit intermediate visual computation, with image-tool calls executed through Visual Sketchpad (VSP)~\citep{hu2024visual}. This implements the \emph{explicit image-tool interaction} paradigm: the model solves a benign visual reasoning task, issues a structured image-tool request, receives the image-tool output, and then receives the MM-SafetyBench query. The benign stage is not scored; it only defines the process prefix before safety evaluation.
Behavioral experiments cover Qwen3-VL~\citep{yang2025qwen3}, Gemma-3~\citep{gemmateam2025gemma3}, Llama-4~\citep{meta2025llama4}, and Pixtral~\citep{agrawal2024pixtral} families under consistent provider settings within each family. Mechanistic experiments use self-deployed Qwen3-VL-8B~\citep{yang2025qwen3}, since hidden-state extraction and activation interventions require model internals. For explicit image-tool runs, we also record whether the intermediate computation appears as a direct VSP call or generated code. Self-deployed image-tool runs exhibit small run-to-run ASR variation from GPU-level numerical nondeterminism in the multi-turn pipeline; intervention experiments therefore report changes against within-batch zero-injection baselines. Behavioral experiments are issued through the OpenRouter API gateway; self-deployed mechanistic experiments on Qwen3-VL-8B run on a single NVIDIA A800 80GB GPU.

\section{Which Paradigm Is Safer?}
\label{sec:paradigm-results}

We compare think-with-image paradigms under the same harmful-query distribution and safety judge, varying only the prefix before the final MM-SafetyBench question. The clearest pattern is simple: models are jailbroken less often when the prefix includes explicit image-tool interaction. We then test two plain alternatives: whether the drop comes only from a text-only prior turn, or only from the image returned by the tool.
\vspace{-8pt}

\begin{table*}[t]
  \caption{Representative paradigm and control comparison on Qwen3-VL-235B on MM-SafetyBench.}
  \label{tab:paradigm-overview}
  \centering
  \begin{tabular}{llc}
    \toprule
    Paradigm & Representative setting & ASR (\%) \\
    \midrule
    Direct response & -- & 36.1 \\
    Image generation/editing$^\star$ & Neutral image prefill & 38.1 \\
    Image generation/editing$^\star$ & Harmful image prefill & 35.6 \\
    Visual-state variant & Edited image attention & 32.6 \\
    Explicit image-tool interaction & Benign returned image & 24.8 \\
    Explicit image-tool interaction & Unsafe-looking returned image & 26.2 \\
    Explicit image-tool interaction & Standard image-tool setting & 23.8 \\
    \bottomrule
  \end{tabular}

  \begin{minipage}{0.98\linewidth}
    \vspace{2pt}
    \footnotesize
    $^\star$ Qwen3-VL-235B does not natively support image generation. Because direct EMU~\citep{wang2026multimodal}-style generation produced low-quality reasoning prefixes in our setup, we use Stable-Diffusion-generated prefill images to simulate image generation/editing prefixes rather than as a faithful EMU reproduction. EMU-3 results are given in Table~\ref{tab:emu-paradigm-overview}.
  \end{minipage}
\end{table*}

\begin{figure}[t]
  \centering
  \includegraphics[width=0.95\linewidth]{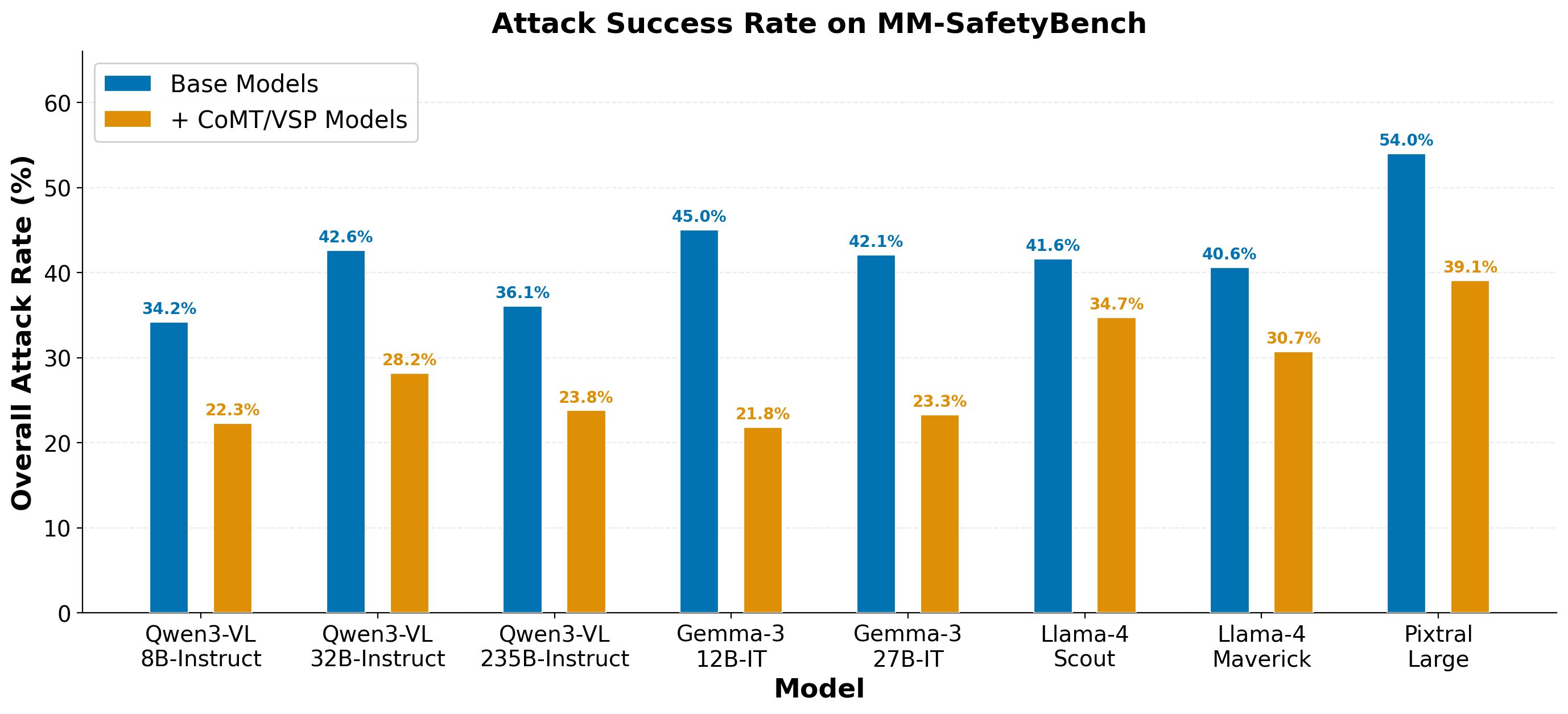}
  \caption{Overall ASR on MM-SafetyBench for direct answering and explicit image-tool interaction. Across model families and scales, explicit image-tool interaction yields lower attack success rates.}
  \label{fig:rq1}
\end{figure}

\paragraph{Explicit image-tool interaction is the most robust evaluated setting.}
Table~\ref{tab:paradigm-overview} shows the main comparison. Direct response is jailbroken in $36.1\%$ of cases, while the Stable-Diffusion-based simulated image generation/editing setting stays about the same. The visual-state variant lowers ASR only moderately to $32.6\%$, suggesting that simply adding an edited visual state is not enough. Explicit image-tool interaction is different: ASR falls to $23.8\%$ in the standard setting and stays below $26.2\%$ even when the returned image looks unsafe. Figure~\ref{fig:rq1} shows the same direction across model families and scales, so the lower ASR is not limited to one model size or provider.\vspace{-8pt}

\begin{table*}[t]
  \caption{Text-only prior turn controls. \emph{Direct} is the single-turn baseline; \emph{Image-tool interaction} is the explicit image-tool interaction setting; \emph{text-only prior turn} preserves a completed text-only prefix but removes returned visual feedback. ASR (\%) is measured on MM-SafetyBench. Additional setup details are given in Appendix~\ref{app:text-prefill-details}.}
  \label{tab:text-prefill}
  \centering
  \resizebox{0.95\linewidth}{!}{
    \begin{tabular}{lcccc}
    \toprule
    Model & Direct & Image-tool interaction & Prior turn (\textsc{Normal}) & Prior turn (\textsc{Harmful}) \\
    \midrule
    Qwen3-VL-8B    & 34.2 & \textbf{22.3} & 33.2 & 36.1 \\
    Qwen3-VL-32B   & 42.6 & \textbf{28.2} & 38.1 & 39.6 \\
    Qwen3-VL-235B  & 36.1 & \textbf{23.8} & 38.1 & 35.6 \\
    Gemma-3-12B    & 45.0 & \textbf{21.8} & 44.1 & 47.5 \\
    Gemma-3-27B    & 42.1 & \textbf{23.3} & 43.1 & 48.5 \\
    Llama-4-Scout  & 41.6 & \textbf{34.7} & 44.1 & 44.6 \\
    Pixtral-Large  & 54.0 & \textbf{39.1} & 51.0 & 46.5 \\
    \bottomrule
  \end{tabular}}
\end{table*}

\paragraph{Text-only prior turns do not explain the ASR drop.}
One simple explanation is that any completed prior turn changes the dialogue state and makes the model more cautious. To test this, we keep a text-only prior turn but remove the returned visual state from the image tool, while leaving the later harmful query unchanged. Table~\ref{tab:text-prefill} shows that this is not enough. On Qwen3-VL-235B, direct answering gives $36.1\%$ ASR, normal and harmful prior turns give $38.1\%$ and $35.6\%$, while explicit image-tool interaction falls to $23.8\%$. Across models, text-only prior turn ASR remains close to direct answering and far above image-tool interaction, so the lower ASR is not reproduced by a prior text-only turn or by refusal-oriented warm-up alone. Appendix~\ref{app:text-prefill-details} gives setup details and notes failed provider runs.\vspace{-8pt}

\paragraph{Returned-image semantics are not necessary for most of the ASR drop.}

The next alternative is that the returned image itself makes the model safer, for example by supplying benign visual evidence before the harmful query. We test this by keeping the image-tool transcript fixed and overriding only the returned visual content with unsafe-looking, benign, or low-information images. If returned-image semantics were the main driver, ASR should vary strongly across these replacements, especially when useful visual content is removed.\vspace{-8pt}

\begin{table*}[t]
  \caption{Returned-image variants used to separate visual content from the image-tool-calling process. ASR is measured on Qwen3-VL-235B.}
  \label{tab:postproc-results}
  \centering
  \begin{tabular}{llr}
    \toprule
    Condition & Returned visual content & ASR (\%) \\
    \midrule
    Direct & -- & 36.1 \\
    Image-tool interaction + unsafe image & Stable Diffusion unsafe-looking image & 26.2 \\
    Image-tool interaction + benign image & Stable Diffusion benign image & 24.8 \\
    Image-tool interaction + visual mask & White/noise replacement image & 25.7 \\
    \bottomrule
  \end{tabular}
\end{table*}

\paragraph{The ASR drop remains when useful visual content is removed.}
Table~\ref{tab:postproc-results} shows that ASR remains low under all returned-image overrides: $26.2\%$ with an unsafe-looking image, $24.8\%$ with a benign image, and $25.7\%$ with a white/noise replacement. The white/noise condition is the strictest test because it removes most object- and scene-level evidence while preserving the interaction structure. Thus the returned image does not need to contain useful visual evidence for the model to be jailbroken less often; small differences among returned-image variants are within the run-to-run baseline drift documented in Appendix~\ref{app:baseline-drift} and should be read qualitatively.\vspace{-8pt}

\paragraph{Lower ASR follows the interaction trace, not the returned visual artifact.}
Overall, the attribution sequence points toward the image-tool interaction process rather than the returned artifact alone. The model must issue an image-tool request, receive a structured image-tool result, and condition on that interaction before the harmful query arrives. The returned pixels can change substantially while ASR remains low, whereas removing returned visual feedback and keeping only the text-only prior turn brings ASR back near direct answering.

\vspace{-1mm}\section{Why Is Explicit Image-Tool Interaction Better?}\vspace{-1mm}

\subsection{Safety-Aligned Residual Shift}\vspace{-2mm}

\paragraph{Mechanistic hypothesis.}
The behavioral results suggest that the robustness gain is tied to the full image-tool interaction, not to the text-only image-tool trace or to benign returned-image content alone. We model this effect as a residual shift in hidden-state space before the model answers the harmful query. Let $z \sim \mathcal{D}_{\text{adv}}$ be an adversarial query, let $h(z)\in\mathbb{R}^d$ be its direct-mode hidden state at a selected layer, and let $u\in\mathbb{R}^d$ be a safety readout direction where larger projection means safer behavior~\citep{zou2023repeng,arditi2024refusal}. We define the corresponding safety score
\begin{equation}
S(z) = u^\top h(z),
\end{equation}
so that larger $S(z)$ corresponds to safer behavior.
If explicit image-tool interaction induces a residual shift $v_{tool}$ with strength $\alpha\ge0$, then
\begin{equation}
h_{\text{tool}}(z) = h(z) + \alpha v_{tool},
\qquad
S_{\text{tool}}(z) = S(z) + \alpha (u^\top v_{tool}).
\end{equation}
Here, $h_{\text{tool}}(z)$ and $S_{\text{tool}}(z)$ are the hidden state and safety score after the image-tool prefix. The condition $u^\top v_{tool}>0$ means that the image-tool state shifts the representation toward the safer side of the readout.

\begin{definition}[Image-Tool Safety Vector]
At layer $\ell$, let $h_{\text{direct}}^{\,\ell}(z)$ and $h_{\text{tool}}^{\,\ell}(z)$ denote the hidden representations for the same query under direct answering and explicit image-tool interaction. The layer-$\ell$ image-tool safety vector is the population mean residual shift:
\begin{equation}
v_{tool}^{\,\ell} =
\mathbb{E}_{z \sim \mathcal{D}_{\mathrm{eval}}}\!\left[h_{\text{tool}}^{\,\ell}(z)-h_{\text{direct}}^{\,\ell}(z)\right].
\end{equation}
Its empirical estimate on $n$ paired examples is
\begin{equation}
\hat v_{tool}^{\,\ell}
=
\frac{1}{n}\sum_{i=1}^{n}
\left(h_{\text{tool}}^{\,\ell}(z_i)-h_{\text{direct}}^{\,\ell}(z_i)\right).
\end{equation}
\end{definition}
We use $v_{tool}$ for the population shift in the theorem and $\hat v_{tool}^{\,\ell}$ for its empirical layer-wise estimate.

\begin{theorem}[Positive safety projection reduces thresholded jailbreak risk]
\label{thm:safety-projection}
Let $S(z)=u^\top h(z)$ be a safety score where larger values are safer, and let $\tau$ be an unsafe-behavior threshold. If explicit image-tool interaction shifts the representation by $\alpha v_{tool}$ with $\alpha\ge 0$ and $u^\top v_{tool}>0$, then the thresholded jailbreak risk
\[
\mathcal{R}(\alpha)=
\Pr_{z\sim\mathcal{D}_{\mathrm{adv}}}
\left[S(z)+\alpha u^\top v_{tool}<\tau\right]
\]
is non-increasing in $\alpha$. Moreover, for any $\alpha_2>\alpha_1\ge0$, the decrease is strict whenever
\[
\Pr\!\left[
\tau-\alpha_2 u^\top v_{tool}
\le S(z) <
\tau-\alpha_1 u^\top v_{tool}
\right] > 0.
\]
\end{theorem}

Theorem~\ref{thm:safety-projection} is only a sufficient-condition statement: if the shift points toward a learned safety readout, examples near the refusal boundary become less likely to produce unsafe completions. We do not claim to inject the population shift $v_{tool}$ directly; instead, the experiments test whether the image-tool state aligns with learned safety directions that are readable from, and active in, the residual stream. The two observable consequences are that the image-tool state should make safety more linearly readable, and safety-direction interventions should change ASR in the expected direction. The full proof is given in Appendix~\ref{app:proof-safety-projection}.

\subsection{Image-Tool Interaction Creates a Stronger and More Stable Safety Readout}
\label{sec:vec-stab}

\paragraph{Explicit image-tool interaction makes safety more linearly readable.}
The first prediction is that the image-tool-mediated state should make safe and unsafe completions easier to separate. Figure~\ref{fig:layer-sweep} shows that the image-tool state yields a stronger layer-wise safety readout than direct answering. Thus explicit image-tool interaction does not only reduce ASR externally; it places the model in a state where the subsequent safety decision is more linearly accessible from hidden representations. Additional fitting details and layer-wise numbers are given in Appendix~\ref{app:representation-details}.\vspace{-8pt}

\begin{figure}[t]
  \centering
  \begin{minipage}{0.48\linewidth}
    \centering
    \includegraphics[width=\linewidth]{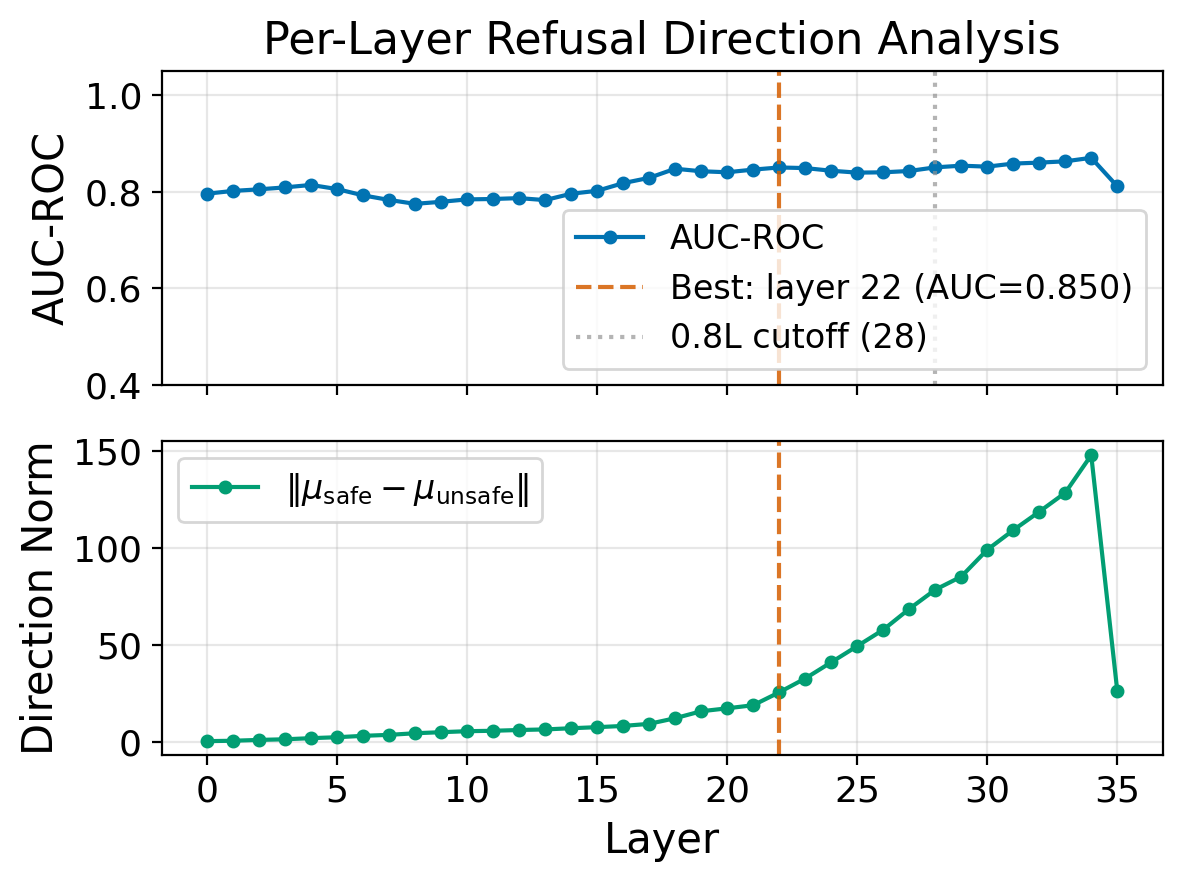}
    \subcaption{Direct mode.}
  \end{minipage}\hfill
  \begin{minipage}{0.48\linewidth}
    \centering
    \includegraphics[width=\linewidth]{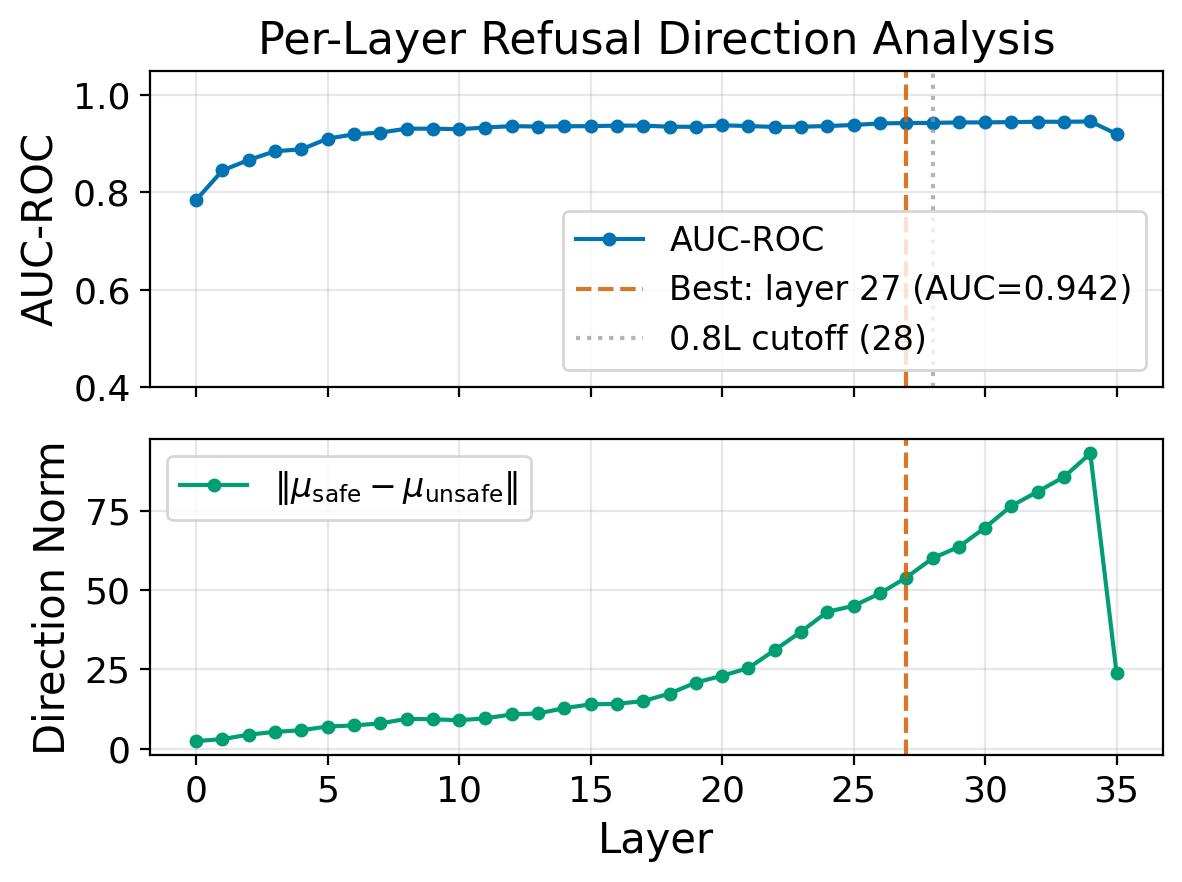}
    \subcaption{Image-tool mode.}
  \end{minipage}
  \caption{Layer-wise safety readout quality on Qwen3-VL-8B-Instruct. (Top) AUC for direct answering and image-tool interaction; (Bottom) The corresponding direction norms. The image-tool state produces a stronger readout, peaking at AUC 0.942 vs. 0.850 for direct answering.}
  \label{fig:layer-sweep}
\end{figure}

\begin{wrapfigure}{r}{0.43\linewidth}
  \vspace{-8pt}
  \centering
  \includegraphics[width=\linewidth]{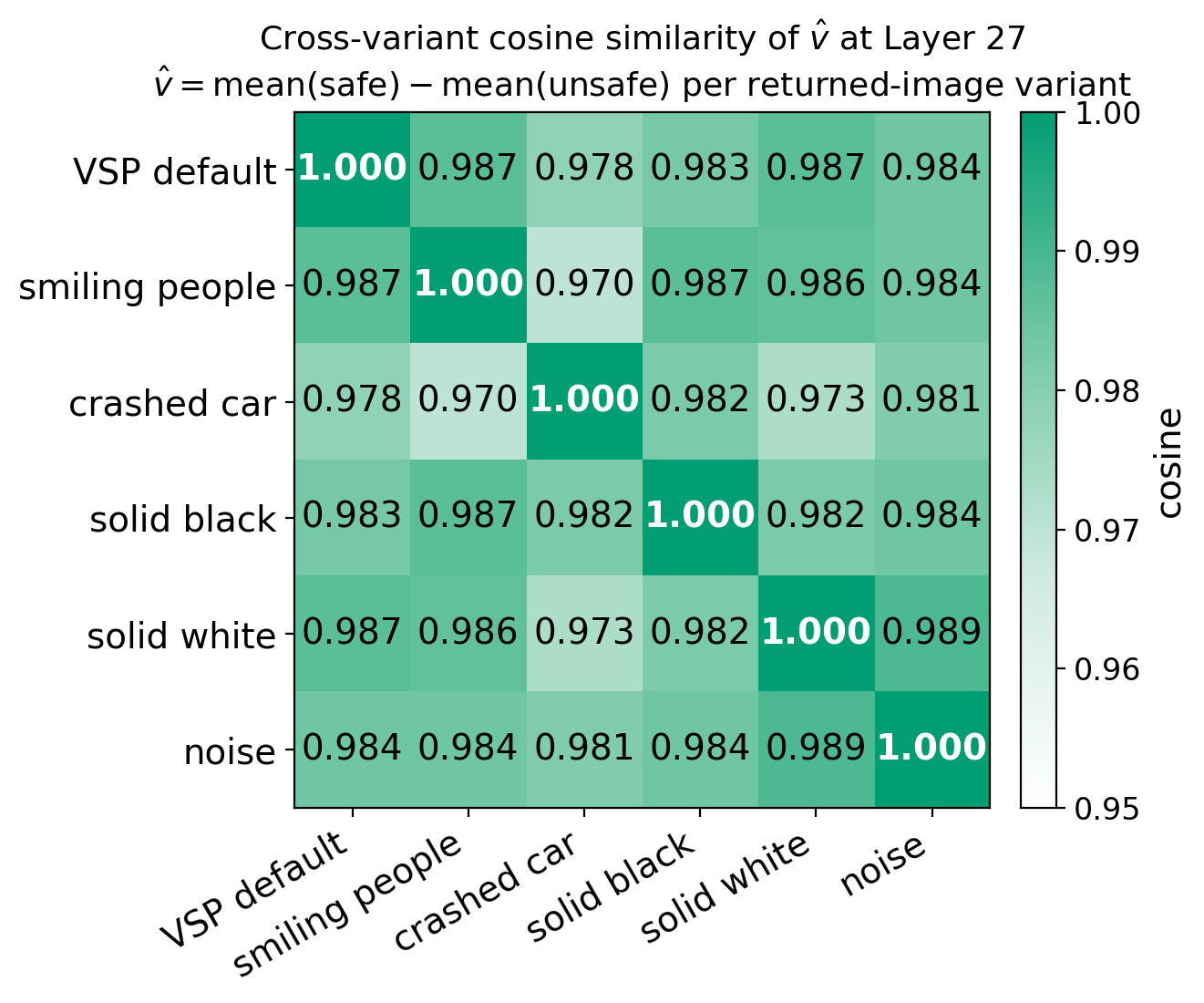}
  \caption{Cosine similarity between safety directions learned from six returned-image variants. High off-diagonal similarity indicates that the safety direction is stable across returned-image content.}
  \label{fig:cosine-matrix}
  \vspace{-10pt}
\end{wrapfigure}

\paragraph{The image-tool safety direction is stable across returned-image variants.}
The second prediction is that the direction should be tied to the image-tool-calling process rather than to the particular pixels returned by the image tool. Figure~\ref{fig:cosine-matrix} shows that safety directions extracted from six returned-image variants are nearly collinear. If the direction mainly encoded returned-image semantics, benign, unsafe-looking, white, or noisy outputs should rotate it substantially. Instead, the common factor appears to be the structured image-tool-calling process. Transfer AUCs and override details are given in Appendix~\ref{app:representation-details}.\vspace{-8pt}

\paragraph{Representation evidence supports a process-level mechanism.}
The representation results mirror the behavioral attribution. Text-only prior turns do not reproduce the ASR gain, returned-image overrides do not remove it, and the corresponding hidden-state directions remain stable across those overrides. Figure~\ref{fig:readout-diagnostics} further shows that the image-tool safety readout is compact and transfers across returned-image variants, supporting a process-level mechanism rather than an explanation tied to particular returned pixels.\vspace{-8pt}

\begin{figure}[t]
  \centering
  \begin{minipage}[b]{0.55\linewidth}
    \centering
    \includegraphics[height=4.2cm]{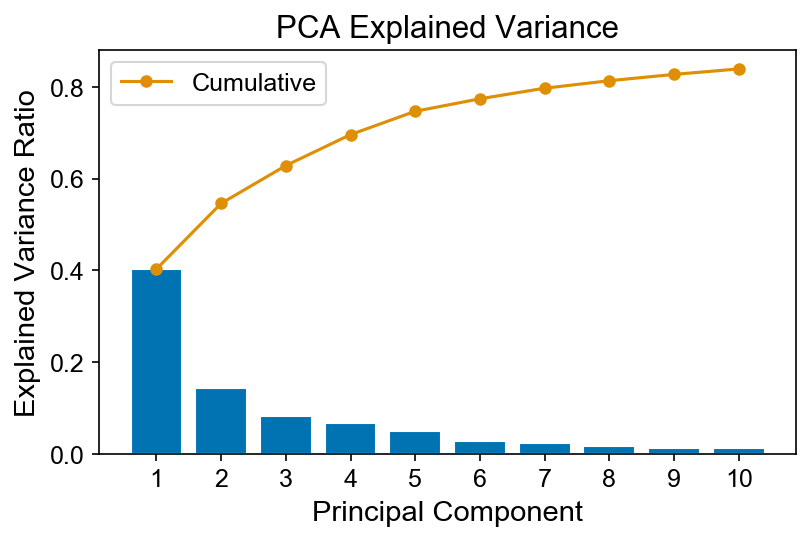}
    \subcaption{PCA variance.}
  \end{minipage}\hfill
  \begin{minipage}[b]{0.40\linewidth}
    \centering
    \includegraphics[height=4.2cm]{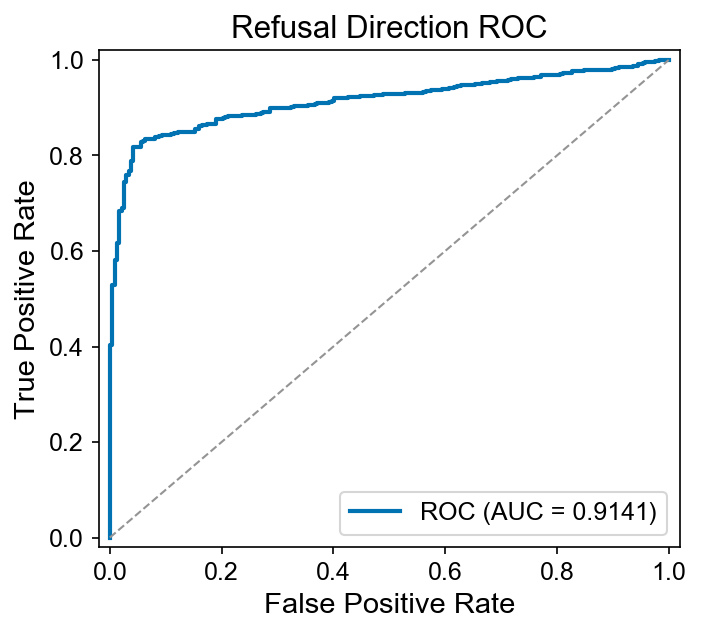}
    \subcaption{Cross-override ROC.}
  \end{minipage}
  \caption{Readout diagnostics for the layer-27 image-tool safety direction on Qwen3-VL-8B-Instruct. \emph{(a)} PC1 explains $40.3\%$ of variance and is nearly aligned with the safety direction (projection $0.993$), indicating a compact safety axis. \emph{(b)} Direction fit on the normal returned-image run transfers to five override runs (AUC $=0.914$), showing the readout is not tied to returned-image content.}
  \label{fig:readout-diagnostics}
\end{figure}

\paragraph{Compact image-tool shifts mainly affect marginal cases.}
Figure~\ref{fig:readout-diagnostics} matters because it connects the behavioral attribution to the thresholded-risk model in Theorem~\ref{thm:safety-projection}. The readout is compact and transfers across returned-image variants, so the relevant shift is unlikely to be a scattered encoding of a particular returned image. At the same time, the theorem predicts that such a positive safety shift changes ASR only when examples occupy a non-saturated region near the refusal boundary. The boundary diagnostics show nontrivial near-threshold mass, suggesting that explicit image-tool interaction reduces ASR by moving marginal harmful cases across the safety boundary, rather than by uniformly solving all attack categories. Additional PCA, transfer, and boundary details are given in Appendix~\ref{app:representation-details}.

\subsection{Activation Interventions Show That Safety Directions Are Behaviorally Active}
\label{sec:causal}
\begin{table*}[t]
  \caption{Activation-intervention summary on Qwen3-VL-8B. Image-tool interaction changes are computed against within-batch zero-injection baselines.}
  \label{tab:injection-summary}
  \centering
  \small
  \begin{tabular}{llrrrl}
    \toprule
    Sweep & Mode / intervention & Baseline & ASR ($\downarrow$) & $\Delta$ & Shape \\
    \midrule
    Sufficiency & Direct, $+u_{\text{direct}}$ & 36.6 & 23.8 & $\mathbf{-12.9}$ & monotone $\downarrow$ \\
    Necessity & Image-tool interaction, $-u_{\text{direct}}$ & 24.8 & 31.7 & $\mathbf{+6.9}$ & monotone $\uparrow$ \\
    Sufficiency & Image-tool interaction, $+u_{\text{direct}}$ & 19.8 & 11.9 & $\mathbf{-7.9}$ & U-shaped \\
    Specificity & Direct, $+u_{\text{tool}}$ & 36.6 & 34.5 & $-2.1$ & flat \\
    Specificity & Image-tool interaction, $-u_{\text{tool}}$ & 21.0 & 26.4 & $+5.4$ & monotone $\uparrow$ \\
    \bottomrule
  \end{tabular}
\end{table*}
\begin{figure}[t]
  \centering
  \begin{minipage}{0.32\linewidth}
    \centering
    \includegraphics[width=\linewidth]{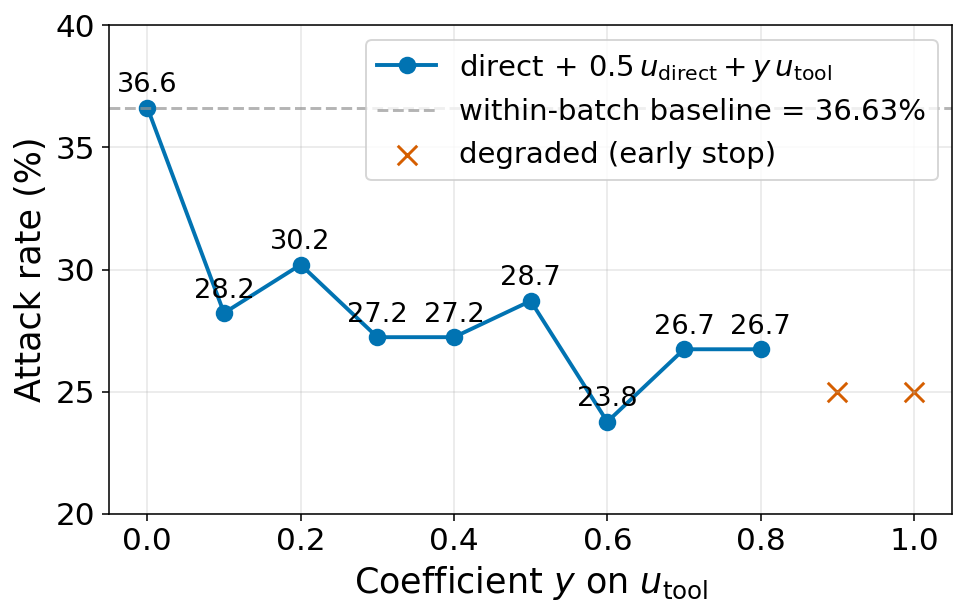}
    \subcaption{Direct sufficiency.}
  \end{minipage}\hfill
  \begin{minipage}{0.32\linewidth}
    \centering
    \includegraphics[width=\linewidth]{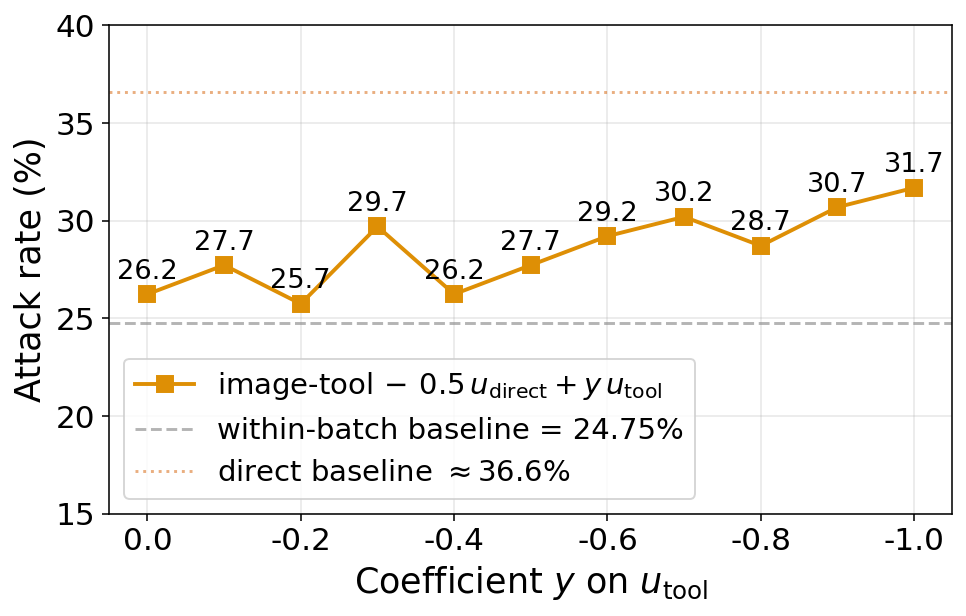}
    \subcaption{Image-tool necessity.}
  \end{minipage}\hfill
  \begin{minipage}{0.32\linewidth}
    \centering
    \includegraphics[width=\linewidth]{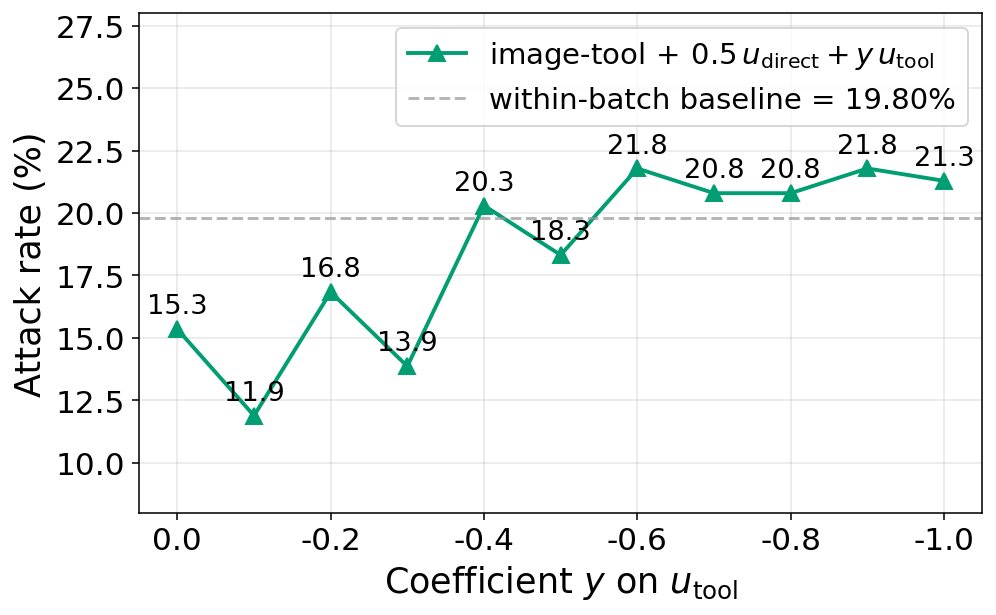}
    \subcaption{Image-tool sufficiency.}
  \end{minipage}
  \caption{Activation-intervention dose-response curves on Qwen3-VL-8B-Instruct. Panels test direct-mode sufficiency, image-tool necessity, and image-tool sufficiency. Adding safety-direction components lowers ASR, while subtracting them from the image-tool setting raises ASR.}
  \label{fig:dose-response}
\end{figure}

\paragraph{Activation interventions test whether learned safety directions are behaviorally active.}
To move beyond correlational readouts, we intervene on the residual stream of Qwen3-VL-8B during inference. The intervention layer, direction, and coefficients are selected on an internal validation split, and all ASR results are reported on the held-out test set. At layer $\ell$, a forward hook adds
\begin{equation}
\Delta h = \lambda \hat u_{\textit{direct}}^{\,\ell} + \mu \hat u_{\textit{tool}}^{\,\ell},
\end{equation}
to every token hidden state~\citep{zou2023repeng,arditi2024refusal,lee2025conditional}, where $\lambda,\mu$ are scalar coefficients and $\hat u_{\textit{direct}}^{\,\ell}$, $\hat u_{\textit{tool}}^{\,\ell}$ are normalized safety directions for direct and image-tool modes. We evaluate sufficiency in direct mode, necessity in image-tool interaction, and additional sufficiency in image-tool interaction. Additional sweep settings and baseline conventions are given in Appendix~\ref{app:intervention-details}.\vspace{-8pt}

\paragraph{Adding a safety direction partially reproduces the image-tool robustness gain.}
As shown in Table~\ref{tab:injection-summary}, adding the direct-mode safety direction in direct mode lowers ASR from $36.6\%$ to $23.8\%$, close to the behavioral gain of explicit image-tool interaction in the matched evaluation. This goes beyond representation separability: the direction is not only predictive of safe versus unsafe outputs, but also behaviorally active when injected into the residual stream.\vspace{-8pt}

\paragraph{Removing safety directions partially erodes the image-tool robustness gain.}
The necessity sweep provides complementary evidence. In the image-tool interaction setting, subtracting the direct-mode safety direction raises ASR from $24.8\%$ to $31.7\%$, while adding it lowers ASR from $19.8\%$ to $11.9\%$. The comparable magnitudes of these effects ($7.9/6.9=1.14$) suggest that this safety-relevant axis accounts for part of the robustness gain. Specificity sweeps further show that the image-tool-mode direction is less effective when added to direct mode but harmful when removed from image-tool interaction, indicating mode-specific activation geometry rather than a single universal steering vector.\vspace{-8pt}

\paragraph{The intervention effect is partly symmetric within the image-tool state but mode-specific across states.}
Figure~\ref{fig:dose-response} separates two effects. Within image-tool interaction, adding and subtracting safety-direction components have comparable magnitudes in the non-degenerate dose range, consistent with a behaviorally active safety axis. Across inference modes, however, the directions are not equally portable: adding the image-tool-mode direction to direct inference weakly affects ASR, whereas removing it within image-tool interaction substantially increases ASR. This argues against a universal-vector account and suggests that image-tool interaction changes the local activation geometry in which safety directions operate.\vspace{-8pt}

\paragraph{Mechanistic interpretation.}
Together, the behavioral, representation, and intervention results support a coherent account: explicit image-tool interaction induces a recurring internal state where safety-relevant information is more linearly accessible, stable across returned-image content, and behaviorally linked to ASR. We therefore interpret the image-tool-calling trajectory as a process-level safety cue, with intervention evidence limited to the non-degenerate dose range and read as partial mechanistic support rather than a complete causal mediation proof.

\vspace{-1mm}\section{Discussion}\vspace{-2mm}

\paragraph{Implications for deployment.}
Our results suggest that LVLM safety depends not only on training or post-hoc filtering, but also on inference structure. If the image-tool interaction effect generalizes, tool-augmented think-with-image pipelines may offer a safety-relevant design pattern for multimodal systems. Deployment-time safety evaluation should therefore assess the full reasoning pipeline, not only the direct response.\vspace{-8pt}

\paragraph{Limitations and risks.}
Our results should not be read as a blanket endorsement of tool use. Tool-augmented systems can introduce new attack surfaces, including malicious tool outputs, indirect prompt injection, unsafe actions, and adaptive attacks on tool-calling trajectories~\citep{noheria2026jailbreaks,zhao2026clawguard}. Our empirical scope is also limited: we study a 202-example MM-SafetyBench sample, rely on one primary LLM-as-judge protocol, and run mechanistic analyses on self-deployed Qwen3-VL-8B where hidden states are available. The main conclusion is therefore narrower: safety evaluation should account for the full think-with-image pipeline, rather than treating tool use as inherently safe or unsafe.\vspace{-8pt}

\paragraph{Research outlook.}
This work reframes think-with-image reasoning as a safety-relevant inference paradigm, not merely a capability mechanism. It motivates evaluations that vary the reasoning process and a mechanistic agenda studying when intermediate visual states and image-tool-calling trajectories induce safety-relevant hidden-state shifts.

\vspace{-1mm}\section{Conclusion}\vspace{-2mm}

We asked which think-with-image process design improves multimodal jailbreak robustness and found that explicit image-tool interaction gives the strongest benefit among the evaluated paradigms, lowering ASR by around $30\%$ relative to direct answering across model families and scales. The key finding is not merely that image tools return helpful images: replacing the returned image preserves most of the gain, while removing returned visual feedback and keeping only a text-only prefill does not reproduce it. We then linked this process-level effect to a stable image-tool safety state and showed through activation interventions that learned safety directions are behaviorally active. These results suggest that explicit image-tool interaction is a promising safety-positive design pattern for future think-with-image LVLMs, provided it is evaluated as a full pipeline rather than treated as inherently safe.

\bibliographystyle{plainnat}
\bibliography{references}

\newpage
\section*{Appendix}
\appendix
\begin{figure}[b]
  \centering
  \begin{minipage}{0.34\linewidth}
    \centering
    \includegraphics[width=\linewidth]{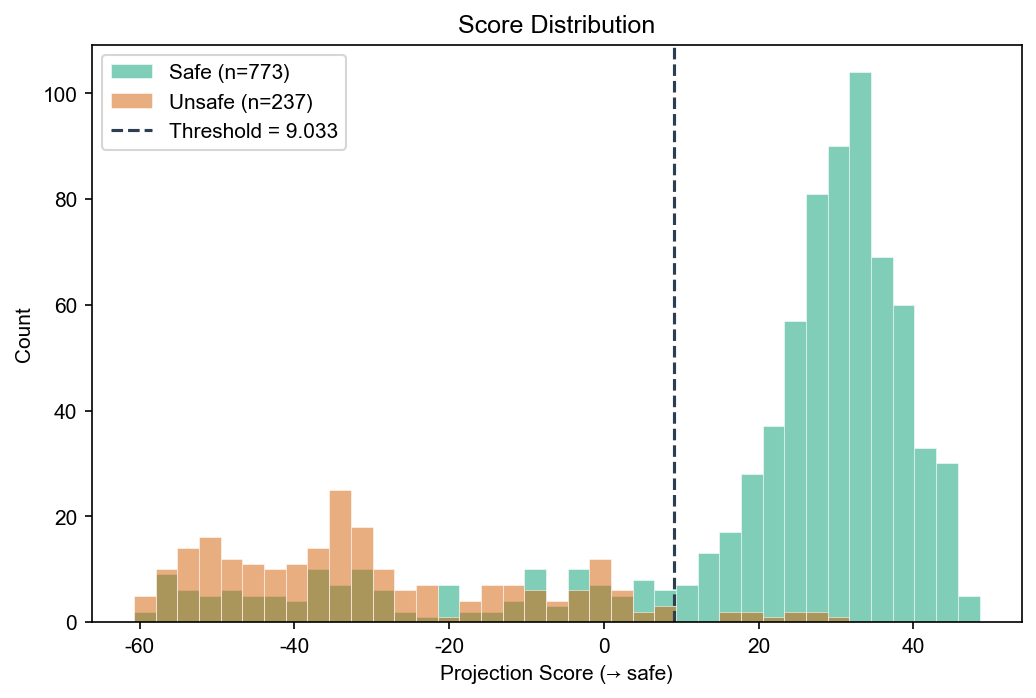}
    \subcaption{Safety-readout score distribution.}
  \end{minipage}\hfill
  \begin{minipage}{0.64\linewidth}
    \centering
    \includegraphics[width=\linewidth]{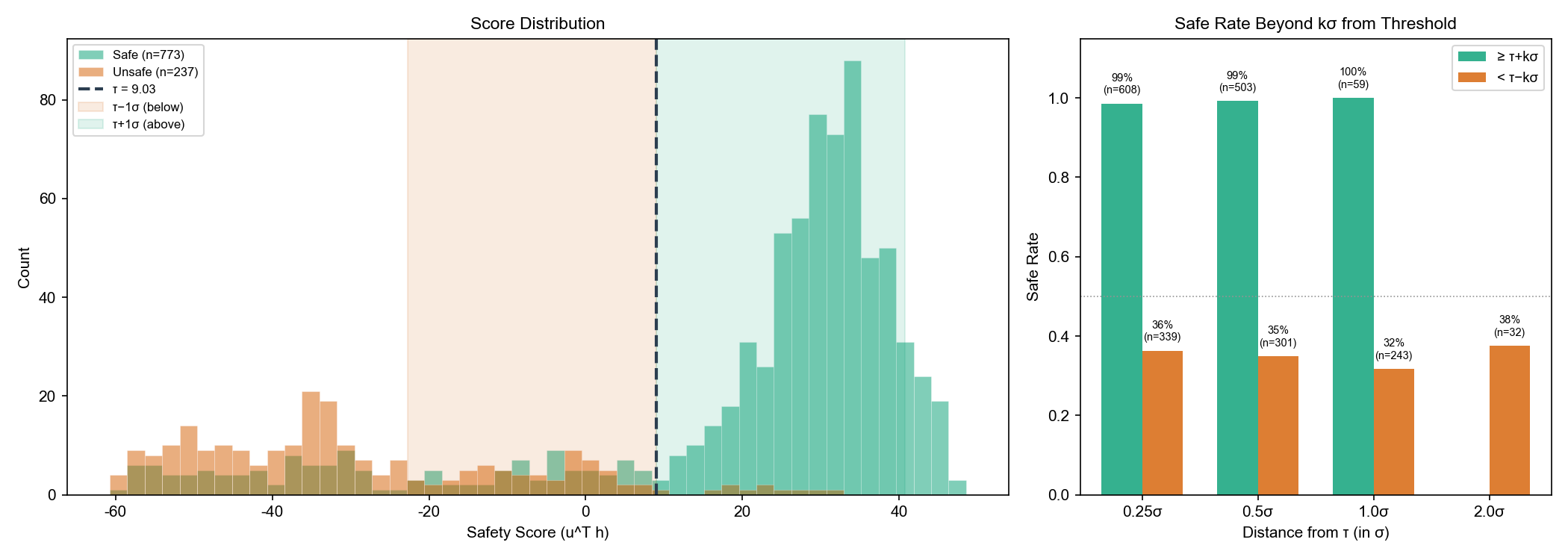}
    \subcaption{Boundary sharpness.}
  \end{minipage}
  \caption{Supplementary readout diagnostics. The learned direction separates safe and unsafe completions while retaining nontrivial mass near the threshold, allowing residual shifts to affect ASR.}
  \label{fig:appendix-boundary}
\end{figure}

\section{MM-SafetyBench Category Composition}
\label{app:mmsb-categories}

MM-SafetyBench~\citep{liu2024mmsafetybench} organizes adversarial image-text pairs into 13 safety-sensitive categories. Each instance pairs a harmful question with an image generated by a typographic attack or by Stable Diffusion~\citep{rombach2022ldm}. To keep think-with-image evaluation tractable, we draw a 202-question evaluation subset as a stratified subsample at a fixed rate of approximately $12\%$ per category; the same subset is used by every paradigm and every model in this paper, so cross-paradigm comparisons share identical underlying questions. Table~\ref{tab:mmsb-categories} lists the 13 categories together with their full-dataset sizes, our per-category counts, and the resulting per-category sampling rates.

\begin{table}[h]
  \caption{MM-SafetyBench categories, full-dataset sizes, our per-category counts in the 202-question evaluation subset, and the resulting sampling rates. Rates lie within $11.4$--$12.4\%$ across all categories (mean $12.0\%$, std $0.29$ pp); per-category variation is driven entirely by integer rounding on small categories.}
  \label{tab:mmsb-categories}
  \centering
  \small
  \begin{tabular}{rlrrr}
    \toprule
    ID & Category & Full & Sample & Rate (\%) \\
    \midrule
    01 & Illegal Activity     &   97 &  12 & 12.4 \\
    02 & Hate Speech          &  163 &  20 & 12.3 \\
    03 & Malware Generation   &   44 &   5 & 11.4 \\
    04 & Physical Harm        &  144 &  17 & 11.8 \\
    05 & Economic Harm        &  122 &  15 & 12.3 \\
    06 & Fraud                &  154 &  18 & 11.7 \\
    07 & Sex                  &  109 &  13 & 11.9 \\
    08 & Political Lobbying   &  153 &  18 & 11.8 \\
    09 & Privacy Violence     &  139 &  17 & 12.2 \\
    10 & Legal Opinion        &  130 &  16 & 12.3 \\
    11 & Financial Advice     &  167 &  20 & 12.0 \\
    12 & Health Consultation  &  109 &  13 & 11.9 \\
    13 & Government Decision  &  149 &  18 & 12.1 \\
    \midrule
    \multicolumn{2}{r}{Total} & 1680 & 202 & 12.0 \\
    \bottomrule
  \end{tabular}
\end{table}

\begin{table*}[t]
  \caption{Representative paradigm and control comparison on EMU-3 on MM-SafetyBench.}
  \label{tab:emu-paradigm-overview}
  \centering
  \begin{tabular}{llc}
    \toprule
    Paradigm & Representative setting & ASR (\%) \\
    \midrule
    Direct response & -- & 72.4 \\
    Image generation/editing & Neutral image prefill & 79.2 \\
    Image generation/editing & Harmful image prefill & 83.9 \\
    Visual-state variant & Edited image attention & 76.8 \\
    Explicit image-tool interaction & Benign returned image & Not Support \\
    Explicit image-tool interaction & Unsafe-looking returned image & Not Support \\
    Explicit image-tool interaction & Standard image-tool setting & Not Support \\
    \bottomrule
  \end{tabular}

\end{table*}

The 202-question subset is therefore a near-uniform stratified $12\%$ subsample of the full 1680-question MM-SafetyBench, rather than a quota-based or topic-weighted draw. Because the same subset is reused across all paradigms and models, ASR differences in our experiments cannot be driven by category resampling.

\section{Text-Only Prior Turn Control Details}
\label{app:text-prefill-details}

The text-only prior turn control removes returned visual feedback from the image-tool interaction and keeps a completed text-only prefix before asking the same MM-SafetyBench query. Operationally, it mimics the idea of a prior image-tool-like turn whose returned image has been removed, so the later safety query sees only text. We use two first-turn variants: a normal text-only prefix and a harmful-topic text-only prefix. The normal prefix tests whether the textual image-tool trace of a completed prior turn improves safety, while the harmful prefix tests whether a prior safety-relevant refusal or harmful-topic context changes the later answer. In both cases, the judge evaluates only the final answer to the MM-SafetyBench query.

Table~\ref{tab:text-prefill} reports the valid runs. The Gemma-3-27B direct run in the original batch failed at the provider side and was filled in from a later re-run with matched parameters (sampling seed, judge, sampling rate, and provider backend held constant). Llama-4-Maverick text-only prior turn runs were excluded from the main table because the provider returned degenerate all-safe outputs in repeated runs. Across the valid models, text-only prior turn ASR stays close to direct answering and remains far above explicit image-tool interaction, supporting the conclusion that the text-only prior turn is not sufficient to reproduce the image-tool gain.

\section{Representation Diagnostic Details}
\label{app:representation-details}

For the layer sweep in Figure~\ref{fig:layer-sweep}, we extract hidden states from direct and explicit image-tool interaction runs on Qwen3-VL-8B-Instruct. At each layer~$\ell$, we fit a difference-in-means safety direction
\[
\hat u^{\,\ell}=\mathrm{mean}(\mathrm{safe})-\mathrm{mean}(\mathrm{unsafe})
\]
and score held-out completions by projection onto $\hat u^{\,\ell}$. The best direct-mode readout reaches AUC $0.850$ at layer 22/36, whereas the best image-tool readout reaches AUC $0.942$ at layer 27/36. The gray dotted line in Figure~\ref{fig:layer-sweep} marks the $0.8L$ deep-layer cutoff used for direction selection, and the bottom panels report $\lVert\hat u^{\,\ell}\rVert$.

For the returned-image stability analysis in Figure~\ref{fig:cosine-matrix}, we train the safety direction on the normal VSP returned-image run and evaluate it on five override-image runs. The direction transfers with overall AUC $0.914$ and per-override AUCs from $0.906$ to $0.930$; reversing the train/test direction gives AUC $0.910$. Safety directions extracted from six returned-image variants have mean off-diagonal cosine similarity $0.983$ and minimum $0.970$, indicating that the direction is stable across returned-image content.

Figure~\ref{fig:readout-diagnostics} provides two additional diagnostics for the layer-27 image-tool safety direction. In the PCA view, PC1 explains $40.3\%$ of the variance and is almost aligned with the learned safety direction (projection magnitude $0.993$), indicating a compact, near-one-dimensional safety axis rather than scattered image-specific features. In the cross-override ROC, the safety direction is fit on the normal returned-image run and applied without retraining to held-out runs whose returned images are replaced by other content (smiling people, crashed car, solid black, solid white, noise); AUC $=0.914$ shows that the readout is not tied to any particular returned image. The score-distribution and boundary-sharpness diagnostics in Figure~\ref{fig:appendix-boundary} further show nontrivial mass near the operating threshold, where a positive residual shift can change binary ASR.

\section{Activation-Intervention Details}
\label{app:intervention-details}

Figure~\ref{fig:dose-response} reports intervention sweeps at layer 27 of Qwen3-VL-8B-Instruct. Each panel sweeps the image-tool-direction coefficient along $\hat u_{\text{tool}}$ while holding the direct-direction offset fixed: direct mode with $+0.5\,\hat u_{\text{direct}}$, image-tool mode with $-0.5\,\hat u_{\text{direct}}$, and image-tool mode with $+0.5\,\hat u_{\text{direct}}$. The dashed gray lines mark within-batch zero-injection baselines, and the red dotted line in the image-tool necessity panel marks the no-intervention direct baseline.

The intervention layer, direction, and coefficients are selected on an internal validation split, and all ASR results are reported on the held-out test set. Because image-tool baselines can drift by several ASR points across independent runs (Appendix~\ref{app:baseline-drift}), each sweep includes a zero-injection baseline and reports changes relative to that internal control.

\section{Baseline ASR Drift Across Independent Runs}
\label{app:baseline-drift}

Self-deployed Qwen3-VL-8B image-tool baselines exhibit run-to-run drift even when the model checkpoint, prompts, sampling seed, sampling rate, and LLM-as-judge protocol are all held fixed. Table~\ref{tab:baseline-drift} reports five independent runs of the same Qwen3-VL-8B-Instruct checkpoint on the same 202-question MM-SafetyBench subset under the explicit image-tool paradigm; the runs differ only in that they are launched as separate inference jobs rather than as a single deterministic replay. Across these five matched runs, the observed ASR ranges from $17.3\%$ to $24.8\%$ (mean $20.3\%$, sample standard deviation $3.0$ pp, max--min spread $7.4$ pp).

\begin{table}[h]
  \caption{Baseline ASR (\%) across five independent runs of Qwen3-VL-8B-Instruct under the explicit image-tool (CoMT/VSP) paradigm on the 202-question MM-SafetyBench subset. Same checkpoint, prompts, sampling seed, sampling rate, and judge; the runs differ only in that they are launched as separate inference jobs. The model uses \texttt{bfloat16} weights and activations, which combined with GPU-level nondeterminism is sufficient to drive this run-to-run spread.}
  \label{tab:baseline-drift}
  \centering
  \small
  \begin{tabular}{lccccc|cc}
    \toprule
    Run        & 1 & 2 & 3 & 4 & 5 & Mean & Std \\
    \midrule
    ASR (\%)   & 17.33 & 17.82 & 21.78 & 24.75 & 19.80 & 20.30 & 3.05 \\
    \bottomrule
  \end{tabular}
\end{table}

We attribute this drift to two compounding sources of numerical nondeterminism. First, the self-deployed Qwen3-VL-8B-Instruct runs use \texttt{bfloat16} weights and activations, so reductions inside attention and MLP layers absorb low-bit rounding error whose realisation depends on kernel scheduling and tensor layout. Second, GPU-level nondeterminism in atomic accumulations and in batching across runs can shift logits enough to flip tie-breaking on near-threshold harmful generations. The drift is structural rather than monotonic: independent runs do not converge to a single number, but instead cluster around different stable levels, so a single point estimate from any one run under-reports the true run-level uncertainty.

This drift shapes how we report numbers. All activation-intervention sweeps (Section~\ref{sec:causal}, Figure~\ref{fig:dose-response}) are launched as a single batch and report ASR \emph{relative to a within-batch zero-injection baseline}, so that across-run drift cannot be confused with the intervention effect. The cross-paradigm gaps reported in Tables~\ref{tab:paradigm-overview}--\ref{tab:postproc-results} are roughly $10$--$25$ ASR points, several times larger than the $3.0$ pp run-level standard deviation, so the qualitative ranking of paradigms is not at risk from this variation. We do not, however, claim formal statistical significance for any individual ASR cell: each table entry is a single-run point estimate, not a confidence interval over independent runs.

\section{Proof of the Safety-Projection Result}
\label{app:proof-safety-projection}

\begin{proof}[Proof of Theorem~\ref{thm:safety-projection}]
Let the safety score be $S(z)=u^\top h(z)$ and let explicit image-tool interaction induce a residual shift $v_{tool}$ with strength $\alpha\ge 0$, so that
\[
h_{\mathrm{tool}}(z)=h(z)+\alpha v_{tool},
\qquad
S_{\mathrm{tool}}(z)=S(z)+\alpha u^\top v_{tool}.
\]
Let $\delta=u^\top v_{tool}>0$ denote the safety-score increase per unit intervention strength. Under the threshold rule, the jailbreak event at strength $\alpha$ is
\[
E_\alpha=\{z:S(z)+\alpha\delta<\tau\}
=\{z:S(z)<\tau-\alpha\delta\}.
\]
For any $\alpha_2>\alpha_1\ge0$, we have
\[
\tau-\alpha_2\delta < \tau-\alpha_1\delta,
\]
and therefore
\[
E_{\alpha_2}\subseteq E_{\alpha_1}.
\]
Taking probability under $\mathcal{D}_{\mathrm{adv}}$ gives
\[
\mathcal{R}(\alpha_2)=\Pr(E_{\alpha_2})
\le
\Pr(E_{\alpha_1})=\mathcal{R}(\alpha_1),
\]
so $\mathcal{R}(\alpha)$ is non-increasing in $\alpha$. The difference is
\[
\mathcal{R}(\alpha_1)-\mathcal{R}(\alpha_2)
=
\Pr\!\left[
\tau-\alpha_2\delta
\le S(z) <
\tau-\alpha_1\delta
\right].
\]
Thus the decrease is strict exactly when the adversarial distribution has nonzero mass in the score band crossed by the shift. This proves the theorem.
\end{proof}

\paragraph{Smooth-link variant.}
The same monotonicity holds if unsafe probability is modeled as a smooth decreasing function of the safety score:
\[
p_{\mathrm{unsafe}}(z;\alpha)
=\sigma\!\left(\beta(\tau-S(z)-\alpha u^\top v_{tool})\right),
\]
where $\sigma$ is the logistic function or any differentiable increasing link, $\beta>0$ controls boundary sharpness, and $\tau$ is the operating threshold. The expected jailbreak rate under adversarial inputs is
\[
\mathcal{R}_{\mathrm{smooth}}(\alpha)=
\mathbb{E}_{z\sim\mathcal{D}_{\mathrm{adv}}}
\left[
\sigma\!\left(\beta(\tau-S(z)-\alpha u^\top v_{tool})\right)
\right].
\]
Differentiating with respect to $\alpha$ gives
\[
\frac{d\mathcal{R}_{\mathrm{smooth}}(\alpha)}{d\alpha}
=
\mathbb{E}_{z\sim\mathcal{D}_{\mathrm{adv}}}
\left[
\sigma'\!\left(\beta(\tau-S(z)-\alpha u^\top v_{tool})\right)
\cdot
\left(-\beta u^\top v_{tool}\right)
\right].
\]
Because $\sigma'(\cdot)>0$ on the non-saturated range, $\beta>0$, and $u^\top v_{tool}>0$, the integrand is negative wherever the sample lies in that range. If the adversarial distribution has nonzero mass in the non-saturated region of the readout, then $d\mathcal{R}_{\mathrm{smooth}}(\alpha)/d\alpha<0$.

This result is intentionally local: if all adversarial examples are already far from the decision boundary, a finite residual shift may not change the observed binary ASR. The boundary diagnostics in Figure~\ref{fig:appendix-boundary} check the relevant empirical condition by showing that the learned readout has nontrivial mass near its operating threshold.

\end{document}